\documentclass{llncs}
\usepackage[utf8]{inputenc} 
\usepackage[T1]{fontenc}    
\usepackage{booktabs}       
\usepackage{amsfonts}       
\usepackage{nicefrac}       
\usepackage{microtype}      

\usepackage{graphicx}
\usepackage{color}
\usepackage{amssymb,amsmath}
\usepackage{subfig}

\usepackage{relsize}
\usepackage{wrapfig}

\usepackage{theoremref}

\usepackage{tikz}
\usetikzlibrary{arrows,positioning} 
\tikzset{
    >=stealth',
    punkt/.style={
           rectangle,
           rounded corners,
           draw=black, very thick,
           text width=6.5em,
           minimum height=2em,
           text centered},
    pil/.style={
           ->,
           thick,
           shorten <=2pt,
           shorten >=2pt,}
}

\title{Slack and Margin Rescaling as Convex Extensions of Supermodular Functions}

\author{Matthew B.\ Blaschko}

\institute{Center for Processing Speech \& Images\\
Departement Elektrotechniek, KU Leuven\\
Kasteelpark Arenberg 10\\
3001 Leuven, Belgium\\
\email{matthew.blaschko@esat.kuleuven.be}
}

\begin{document}

\maketitle

\begin{abstract}
  Slack and margin rescaling are variants of the structured output SVM, which is frequently applied to problems in computer vision such as image segmentation, object localization, and learning parts based object models.  They define convex surrogates to task specific loss functions, which, when specialized to non-additive loss functions for multi-label problems, yield extensions to increasing set functions.  We demonstrate in this paper that we may use these concepts to define polynomial time convex extensions of arbitrary supermodular functions, providing an analysis framework for the tightness of these surrogates.
This analysis framework shows that, while neither margin nor slack rescaling dominate the other, known bounds on supermodular functions can be used to derive extensions that dominate both of these, indicating possible directions for defining novel structured output prediction surrogates. 
In addition to the analysis of structured prediction loss functions, these results imply an approach to supermodular minimization in which margin rescaling is combined with non-polynomial time convex extensions to compute a sequence of LP relaxations reminiscent of a cutting plane method.  This approach is applied to the problem of selecting representative exemplars from a set of images, validating our theoretical contributions.
\end{abstract}

\section{Introduction}

Structured output support vector machines \cite{Tsochantaridis2005} are commonly applied to a range of structured prediction problems in computer vision.  A key open question is the tightness of the loss surrogate, with negative results from statistical learning theory indicating that neither slack rescaling nor margin rescaling variants lead to statistical consistency \cite{McAllesterPSD2007}.  Nevertheless, due to their good empirical performance, they are frequently applied in practice, indicating the importance of the analysis of these surrogates. 

In this work, we explore strategies analogous to margin and slack rescaling in structured output support vector machines~\cite{Tsochantaridis2005} for set function minimization in general, and substantially advance the theory of supermodular set function minimization as well as provide insights into possible improvements to structured output support vector machines.  In particular, we advance the theory of convex extensions of supermodular extensions by showing an explicit form for the closest convex extension to the convex closure within a specific multiplicative family (Proposition~\ref{thm:SisTightOverMultiplicativeNoCallsToell}), and an additive family (Proposition~\ref{thm:MarginDominatesAdditive}).  
We prove that extension operators derived from modular upper bounds \cite{JegelkaB2011cvpr,iyer2012-min-diff-sub-funcs} on submodular functions dominate both slack and margin rescaling, but computation of these tighter convex extensions do not correspond with supermodular maximization problems.  This in general suggests a trade-off between the tightness of the extension and the tractability of its computation, with only margin rescaling demonstrated to be polynomial time computable.  We develop a minimization strategy where the computation of margin rescaling is used to optimize other surrogates for which polynomial time algorithms are unknown. 
We summarize key theoretical contributions in  Figure~\ref{fig:mainResults} through a partial ordering over convex extensions of set functions considered in this paper using the notion of restricted operator inequality (Definition~\ref{def:RestrictedOperatorInequality}). If a fixed finite set of 
convex extensions does not have a total ordering, the pointwise maximum of these extensions is a 
convex extension that dominates them all.

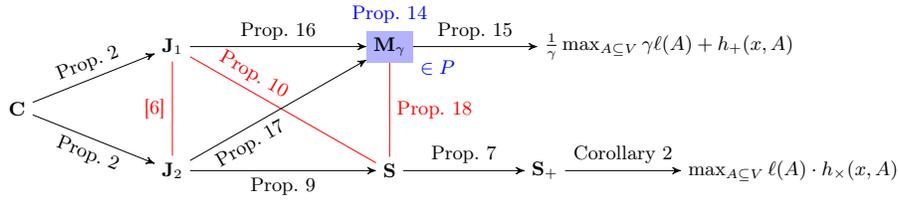
\begin{figure}[t]
\centering
  \resizebox{\textwidth}{!}{
\begin{tikzpicture}[scale=.9, transform shape]
\node (c) at +(0: 2.5) {$\mathbf{C}$};
\node (j1) at +(5,1) {$\mathbf{J}_1$};
\node (j2) at +(5,-1) {$\mathbf{J}_2$};
\node [rectangle, fill=blue!30]
(m) [label=above:{\textcolor{blue}{Prop.~\ref{prop:MisPolyTime}}}] [label=345:{\textcolor{blue}{$\in P$}}] at +(8.5,1) {$\mathbf{M}_\gamma$};
\node (s) at +(8.5,-1) {$\mathbf{S}$};
\node (sp) at +(11,-1) {$\mathbf{S}_+$};
\node (mult) at +(15,-1) {$ \max_{A\subseteq V} \ell(A) \cdot h_{\times}(x,A)$};
\node (add) at +(13,1) {$\frac{1}{\gamma} \max_{A \subseteq V} \gamma \ell(A) + h_{+}(x,A)$};
\draw [->] (c) -- (j1) node[pos=.5,sloped,above] {Prop.~\ref{thm:CisTheGreatestConvexExtension}};
\draw [->] (c) -- (j2) node[pos=.5,sloped,below] {Prop.~\ref{thm:CisTheGreatestConvexExtension}};
\draw [->] (j1) -- (m) node [pos=.5,above] {Prop.~\ref{prop:J1geqM}};
\draw [->] (j2) -- (s) node [pos=.5,below] {Prop.~\ref{prop:J2geqS}};
\draw [->] (s) -- (sp) node [pos=.5,above] {Prop.~\ref{prop:SgeqS+}};
\draw [->] (sp) -- (mult) node [pos=.5,above] {Corollary~\ref{thm:S+dominatesMultiplicativeFamily}};
\draw [red] (j1) -- (j2) node [pos=.5,left] {\cite{iyer2012-min-diff-sub-funcs}};
\draw [red] (j1) -- (s) node[pos=.3,sloped,above] {Prop.~\ref{prop:J1ngeqS}};
\draw [red] (s) -- (m) node [pos=.5,right] {Prop.~\ref{thm:SlackNotDominatesMargin}};
\draw [->] (j2) -- (m) node [pos=.3,sloped,below] {Prop.~\ref{prop:J2geqM}};
\draw [->] (m) -- (add) node [pos=.5,above] {Prop.~\ref{thm:MarginDominatesAdditive}};
\end{tikzpicture}
}
  \caption{A summary of the main theoretical contributions in this paper. 
    $\mathbf{C}$ denotes the convex closure, $\mathbf{J}_1$ and $\mathbf{J}_2$ are extensions based on known modular upper bounds to supermodular functions \cite{JegelkaB2011cvpr,iyer2012-min-diff-sub-funcs}, and $\mathbf{M}$ and $\mathbf{S}$ indicate variants of extensions derived from margin rescaling and slack rescaling, directed edges indicate that the parent dominates the child, and red undirected edges indicate that there is no ordering between a pair of nodes. $\mathbf{M}_\gamma$ is the only convex extension whose computation always corresponds to a supermodular maximization and is therefore known to be polynomial time computable.
}\label{fig:mainResults}
\end{figure}

Supermodular minimization (submodular maximization) has been widely studied both in the context of relaxations as well as for fully combinatorial approximation algorithms~\cite{krause2012survey}.  Notable contributions include a 1/2-approximation algorithm for unconstrained non-negative submodular maximization \cite{Buchbinder2012TLT} and the multi-linear extension \cite{Vondrak2008OAS,Vondrak2010curvature}, which, although non-convex, has been used in an approximate minimization framework.
A large amount of work has been done on relaxations approximate energy minimization in the context of random field models with low maximal clique size \cite{KumarJMLR2009}, including roof duality \cite{Kahl2012GRD}, or by making additional assumptions on the structure of the graph such as balancedness \cite{wellerAISTATS2016}. 
The connection between the structured output SVM \cite{Tsochantaridis2005} and convex extensions of set functions seems to be due to \cite{yu:hal-01151823}.  The comparative difficulty of optimizing slack rescaling vs.\ margin rescaling is known, and an interesting approach to the optimization of slack rescaling via multiple margin rescaling operations can be found in \cite{ChoiAISTATS2016}.
In the sequel, we first develop basic definitions and results for supermodular set functions and LP relaxations (Section~\ref{sec:MathematicalPreliminaries}), we then provide the main theoretical analysis in Section~\ref{sec:SlackAndMarginRescaling} before demonstrating empirical results consistent with the theory (Section~\ref{sec:Empirical}).

\section{Mathematical Preliminaries}\label{sec:MathematicalPreliminaries}

\begin{definition}[Set function]
A set function $\ell$ is a mapping from the power set of a base set $V$ to the reals:
\begin{equation}
\ell : \mathcal{P}(V) \mapsto \mathbb{R} .
\end{equation}
\end{definition}
We will assume wlog that $\ell(\emptyset)=0$ for all set functions in the sequel.
\begin{definition}[Extension]
An extension of a set function is an operator that yields a continuous function over the $p$-dimensional unit cube (where $p=|V|$) such that the function value on each vertex $x$ of the unit cube is equal to $\ell(A)$ for $A$ such that $x_i=1 \iff i \in A$.
\end{definition}

\begin{definition}[Submodular set function]
  A set function $f$ is said to be submodular if for all $A \subseteq B \subset V$ and $x\in V\setminus B$,
  \begin{equation}
    f(A\cup \{x\}) -f(A) \geq f(B \cup \{x\}) - f(B) .
  \end{equation}
\end{definition}
A set function is said to be supermodular if its negative is submodular, and a function is said to be modular if it is both submodular and supermodular.
We denote the set of all submodular functions $\mathcal{S}$, the set of all supermodular functions $\mathcal{G} := \{g | -g\in \mathcal{S}\}$, and $\mathcal{G}_+ := \{g | g(S)\geq 0, \ \forall S\subseteq V\} \cap \mathcal{G}$ denotes the set of all non-negative supermodular functions.  Additional properties of submodular functions can be found in \cite{fujishige2005submodular}.

\begin{lemma}
  Any  $g \in \mathcal{G}$ such that $g(\{x\})\geq 0,\ \forall x\in V$ is an increasing function, i.e.\ $g(A \cup \{x\}) \geq g(A),\ \forall A \subset V, x\in V \setminus A$.
\end{lemma}
\begin{proof}
  We have that $g(\emptyset)=0$ and $g(\{x\})\geq 0,\ \forall x\in V$ by assumption.
From the supermodularity of $g$, $\forall A \subset V, x \in V \setminus A$,
  \begin{align}
  g(A) + \overbrace{g(\{x\})}^{\geq 0} \leq  g(A \cup \{x\}) + \overbrace{g(A \cap \{x\})}^{=0}
 \implies g(A) \leq  g(A \cup \{x\})
  \end{align}
\qed\end{proof}
\begin{corollary}\label{thm:G+increasing}
Any $g\in\mathcal{G}_+$ is an increasing function.
\end{corollary}

\begin{definition}[Convex closure]
The convex closure of a set function $\mathbf{C}\ell$ is defined for all $x\in [0,1]^{|V|}$ as
\begin{align}
\mathbf{C}\ell(x) :=& \min_{\alpha \in \mathbb{R}^{2^{|V|}}} \sum_{A\subseteq V} \alpha_A \ell(A), 
\text{ s.t. } \sum_{A \subseteq V} \alpha_{A} \mathbf{1}_A = x, 
\sum_{A \subseteq V} \alpha_{A} = 1, 
\alpha_{A} \geq 0  \quad \forall A \subseteq V ,
\end{align}
where $\mathbf{1}_A$ denotes the binary vector of length $|V|$ whose entries indexed by the elements in $A$ are $1$.
\end{definition}

In general, the convex closure is not polynomial time solvable, with the notable exception of the Lov\'{a}sz extension for submodular functions \cite{fujishige2005submodular,lovasz1983submodular}.

\begin{definition}[Restricted operator inequality]\label{def:RestrictedOperatorInequality}
For two convex extensions $\mathbf{A}$ and $\mathbf{B}$, we write $\mathbf{A} < \mathbf{B}$ (analogously $\mathbf{A} \leq \mathbf{B}$) if $\mathbf{A}\ell(x) < \mathbf{B}\ell(x) \ \forall \ell, x\in [0,1]^p$.  We may subscript the inequality sign with a class of set functions (e.g.\ $\mathbf{B} \leq_{\mathcal{S}} \mathbf{L}$) if the inequality holds for all set functions in that class.
\end{definition}

\begin{proposition}[Transitivity of restricted operator inequality]
  For two function sets $\mathcal{A}$ and $\mathcal{B}$ and three operators $\mathbf{F}$, $\mathbf{G}$, and $\mathbf{H}$,
  \begin{equation}
    \mathbf{F} \leq_{\mathcal{A}} \mathbf{G} \leq_{\mathcal{B}} \mathbf{H} \implies \mathbf{F} \leq_{\mathcal{A} \cap \mathcal{B}} \mathbf{H} .
  \end{equation}
\end{proposition}

\begin{proposition}\label{thm:CisTheGreatestConvexExtension}
$\mathbf{B} \leq \mathbf{C}$ for all convex extensions $\mathbf{B}$.
\end{proposition}
\begin{proof}
Assume that there exists some convex extension $\mathbf{B}\ell$ such that $\mathbf{B}\ell(x) > \mathbf{C}\ell(x)$ for some $x$ in the $p$-dimensional unit cube.  If we map all $A \subseteq V$ to vertices of the $p$-dimensional unit cube, and identify $\ell$ with $2^p$ distinct points in $\{0,1\}^p \times \mathbb{R}$, then the convex closure is the lower hull of these points taken with respect to the dimension corresponding to the function value.  As each of the $2^p$ points corresponding to the function are vertices of this lower hull, any $\mathbf{B}\ell$ that has a point greater than a corresponding point of $\mathbf{C}\ell$ is not a lower hull and cannot contain all of these vertices, which contradicts the definition of a set function extension.
\qed\end{proof}

\begin{definition}[\cite{JegelkaB2011cvpr,iyer2012-min-diff-sub-funcs}]\label{def:modularBoundsJegelkaIyer}
    The following two operators yield convex extensions for $g \in \mathcal{G}$, but the $\arg\max$ does not correspond to submodular maximization in general:
\begin{align}\label{eq:JegelkaBilmesModularConvexClosure}
  \mathbf{J}_{1}g(x) = \arg\max_{A\subseteq V} g(A) &+ \sum_{i \in V\setminus A} x_i \left( g(A\cup \{i\}) - g(A) \right) \\& - \sum_{i\in A} (1-x_i) \left( g(V) - g(V\setminus \{i\}) \right) , \nonumber
  \\
  \mathbf{J}_{2}g(x) = \arg\max_{A\subseteq V} g(A) & + \sum_{i \in V\setminus A} x_i g(\{i\})
  - \sum_{i\in A} (1-x_i) \left( g(A) - g(A\setminus \{i\}) \right) . 
\end{align}
\end{definition}

\subsection{LP relaxations}

One of the key applications of convex set function extensions is their application to (approximate) minimization of set functions through LP relaxations.  If a relaxation is a convex extension, we have the useful result that an integral minimizer of the relaxation is guaranteed to be an optimal solution to the orignal discrete optimization problem.  In Proposition~\ref{thm:CloserToCmoreLikelyToBeIntegral} we formalize the unsurprising but important result that dominating convex extensions result in integral solutions to LP relaxations with higher probability.  For some set $S \subseteq \mathbb{R}^d$, we denote $\operatorname{int}(S) := [S \cap \{0,1\}^d \neq \emptyset]$, where we have used Iverson bracket notation.  
\begin{proposition}\label{thm:CloserToCmoreLikelyToBeIntegral}
For two operators that yield convex extensions for all set functions $\ell\in\mathcal{A}$, $\mathbf{F}\leq_{\mathcal{A}} \mathbf{G} \leq_\mathcal{A} \mathbf{C}$, and for all distributions $p : \mathcal{A} \mapsto \mathbb{R}_+$,
\begin{align}
\mathbb{E}_{\ell \sim p} \left\{ \operatorname{int}\left(\arg\min_{x\in [0,1]^{|V|}} \mathbf{F}\ell(x) \right) \right\} \leq \mathbb{E}_{\ell \sim p} \left\{ \operatorname{int} \left(\arg\min_{x\in [0,1]^{|V|}} \mathbf{G}\ell(x) \right)  \right\},
\end{align}
where $\arg\min$ is defined as a map to the set of minimizers of an expression.
\end{proposition}
\begin{proof}
It is sufficient to demonstrate that 
\begin{align}
  \operatorname{int}\left(\arg\min_{x\in [0,1]^{|V|}} \mathbf{F}\ell(x) \right)
  \impliedby \operatorname{int}\left(\arg\min_{x\in [0,1]^{|V|}} \mathbf{G}\ell(x) \right).
\end{align}

Denote $x^*$ an integral optimum of $\arg\min_{x\in [0,1]^{|V|}} \mathbf{C}\ell(x)$,  As $\mathbf{C}\ell(x)$ is total dual integral \cite[Section~4.6]{Conforti2014IP}, we have that $\mathbf{C}\ell(x^*) = \min_{A\subseteq V} \ell(A)$.  If $\ell(x^*) = \mathbf{G}\ell(x^*) \neq \min_{x\in [0,1]^{|V|}} \mathbf{G}\ell(x)$, there must be a negative directional gradient $\nabla_{v} \mathbf{G}\ell(x^*)$ for some $v\in \mathbb{R}^{|V|}$ pointing into the unit cube from $x^*$.  From the definition of set function extensions and the fact that $\mathbf{F}\leq_{\mathcal{A}} \mathbf{G}$, we have that $0 > \nabla_{v} \mathbf{G}\ell(x^*) \geq \nabla_{v} \mathbf{F}\ell(x^*)$ and therefore $x^* \notin \arg\min_{x\in [0,1]^{|V|}} \mathbf{F}\ell(x)$.
\qed\end{proof}

\section{Slack and margin rescaling}\label{sec:SlackAndMarginRescaling}

In these sections we develop two convex extensions of increasing set functions that are based on the surrogate loss functions of two variants of the structured output support vector machine \cite{Tsochantaridis2005}. Their relationship to convex extensions of set functions is due to \cite{yu:hal-01151823}, which developed results analogous to Prop.~\ref{thm:SlackRescalingIsConvexExtension},  and~\ref{thm:MarginRescalingScaleFactor}.

\subsection{Slack rescaling}

\begin{definition}[Slack rescaling for increasing set functions \cite{Tsochantaridis2005,yu:hal-01151823,Yu2016a}]\label{def:SlackRescalingIncreasing2}
For an increasing set function $\ell$,
\begin{equation}\label{eq:slackRescalingTighter}
\mathbf{S}_{+}\ell(x) = \max_{A\subseteq V} \ell(A)\left( 1 - |A| + \sum_{i\in A} x_{i} \right) .
\end{equation}
\end{definition}

\begin{proposition}\label{thm:SpolytimeConvexExtension}
Computation of $\mathbf{S}_{+}\ell(x)$ corresponds to a non-supermodular maximization problem for $x\in [0,1]^{|V|}$ and supermodular increasing $\ell$.
\end{proposition}
\begin{proof}
To show that the optimization is not in general a supermodular maximization, we may consider the following counterexample:
$V = \{a,b\}$, $\ell(\emptyset) = \ell(\{a\}) = \ell(\{b\}) = 0$, $\ell(\{a,b\}) = \varepsilon$.  For $\varepsilon>0$ this is strictly supermodular.  Denote
  \begin{align}
    \tilde{\ell}(A):= \ell(A)\left( 1 - |A| + \sum_{i\in A} x_{i} \right) 
  \end{align}
  For $x=\mathbf{0}$, we have that $\tilde{\ell}(\emptyset) = \tilde{\ell}(\{a\}) = \tilde{\ell}(\{b\}) = 0$ and
$\tilde{\ell}(\{a,b\}) = \varepsilon (1 - 2 + 0) = -\varepsilon$, but this indicates that $\ell = -\tilde{\ell}$ and as $\ell$ is not modular they cannot both be supermodular.
\qed\end{proof}

\begin{proposition}[Proposition~1 of~\cite{yu:hal-01151823}]\label{thm:SlackRescalingIsConvexExtension}
Definition~\ref{def:SlackRescalingIncreasing2}\ yields a convex extension for all increasing $\ell$.
\end{proposition}
\begin{proof}
Equation~\eqref{eq:slackRescalingTighter} is convex:
The r.h.s.\ is a maximum over linear functions of $x$, and is therefore convex in $x$.

Equation~\eqref{eq:slackRescalingTighter} is an extension:
We will use the notation $\operatorname{set} : \{0,1\}^{|V|} \rightarrow \mathcal{P}(V)$ to denote the set associated with an indicator vector.
We must have that $\mathbf{S}_{+}\ell(x) = \ell(A)$ whenever $x\in \{0,1\}^{|V|}$ and $\operatorname{set}(x)=A$.  For $\operatorname{set}(x)=A$ 
\begin{equation}
\ell(A) \mathlarger{\mathlarger{\mathlarger{(}}} 1 - |A| + \overbrace{\sum_{i\in A} x_{i}}^{=|A|} \mathlarger{\mathlarger{\mathlarger{)}}} = \ell(A).
\end{equation}
We now must show that when $\operatorname{set}(x)=A$
\begin{align}\label{eq:StighterBiggerThanOtherPlanes}
\ell(A) \geq \ell(B)\left( 1 - |B| + \sum_{i\in B} x_i \right) = \ell(B)\left(1-|B| + |A \cap B| \right).
\end{align}
For any $B\nsubseteq A$: 
$\overbrace{\ell(B)}^{\geq 0}\overbrace{\left(1-|B| + |A \cap B| \right)}^{\leq 0} \leq 0 \leq \ell(A)$, 
and for any $B\subseteq A$ 
$\overbrace{\ell(B)}^{\leq \ell(A)}\overbrace{\left(1-|B| + |A \cap B| \right)}^{= 1} \leq \ell(A)$.
\qed\end{proof}
An example of the extension 
is shown in Figure~\ref{fig:SlackRescalingTighterExample}.

\begin{figure}[t]
\centering
\includegraphics[width=0.6\columnwidth]{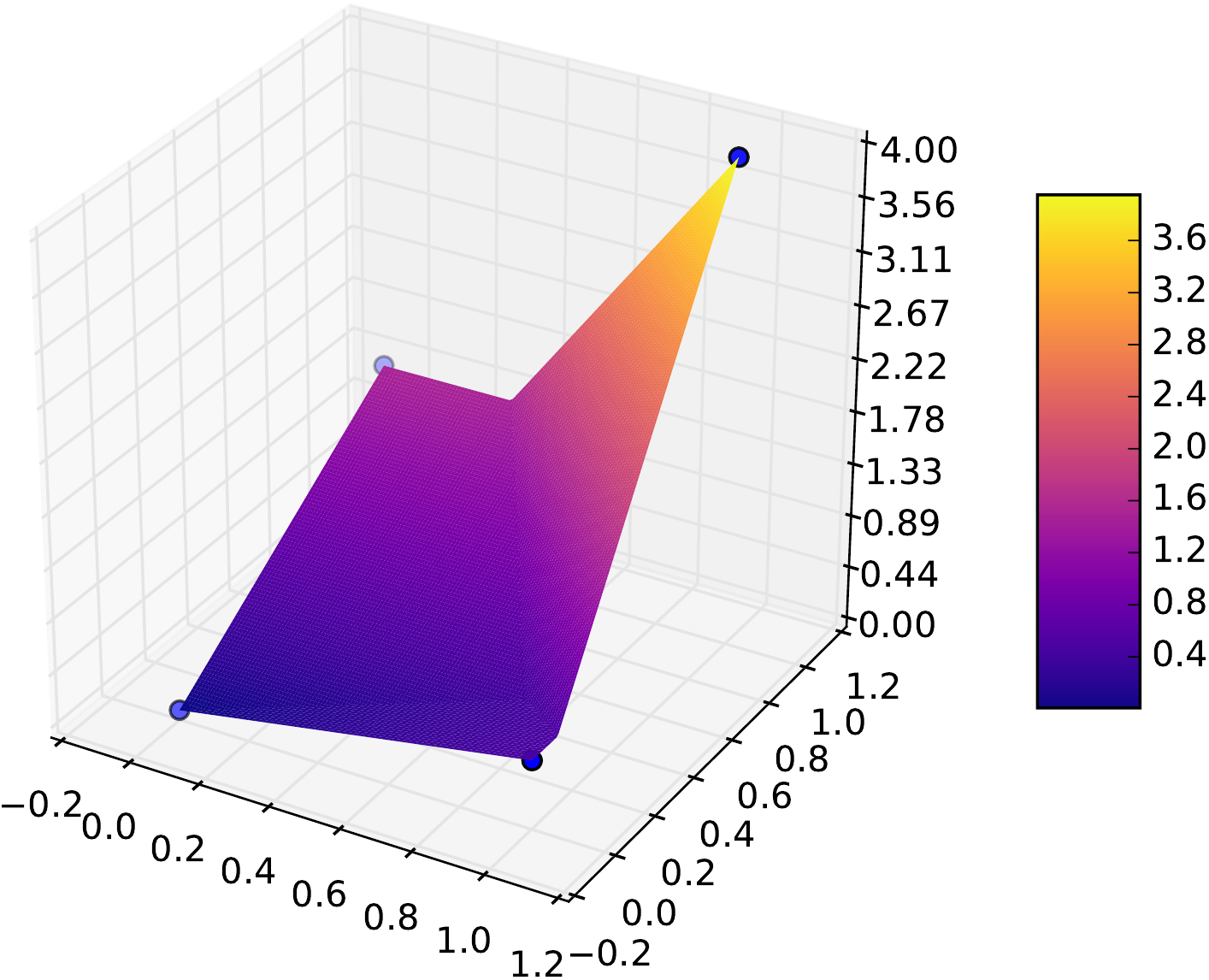}
\caption{An example of the convex extension achieved with Definition~\ref{def:SlackRescalingIncreasing2} applied to the set function $\ell(\emptyset) = 0$, $\ell(\{a\}) = 0.5$, $\ell(\{b\}) =  1.5$, $\ell(\{a,b\}) = 4$.}\label{fig:SlackRescalingTighterExample}
\end{figure}

We can see that Definition~\ref{def:SlackRescalingIncreasing2} is an instance of a general strategy for multiplicative extensions of increasing functions:
\begin{equation}\label{eq:multiplicativeExtension}
\mathbf{B} \ell(x) = \max_{A\subseteq V} \ell(A) \cdot h_{\times}(x,A)
\end{equation}
where $h_{\times} : [0,1]^{|V|} \times \mathcal{P}(V) \mapsto \mathbb{R}$ is a function convex in $x$.   We require $h_{\times}(x,A) \leq \frac{\ell(\operatorname{set}(x))}{\ell(A)}$ for integral $x$ and all $A$ in order to ensure that $\ell(A)h_{\times}(x,A) \leq \ell(\operatorname{set}(x))$, which is the general form of the inequality in Equation~\eqref{eq:StighterBiggerThanOtherPlanes}.  This strategy is quite general if $h_{\times}$ can have a dependency on $\ell$ as the convex closure is recovered simply by setting $h_{\times}(x,A) = \frac{\mathbf{C}\ell(x)}{\ell(A)}$; In Proposition~\ref{thm:SisTightOverMultiplicativeNoCallsToell} we will limit ourselves to $h_{\times}$ that have no explicit dependency on $\ell$ beyond that it is an increasing function.
For increasing $\ell$, $\frac{\ell(\operatorname{set}(x))}{\ell(A)}$ is bounded below by $1$ for all $A \subseteq \operatorname{set}(x)$ and is bounded below by $0$ for all $A \nsubseteq \operatorname{set}(x)$.  If we additionally assume that $\ell \in \mathcal{G}_+$, these bounds remain unchanged.
If we take these bounds with respect to all increasing $\ell$, or indeed all $g \in \mathcal{G}_+$, the bounds are sharp.
\begin{proposition}\label{thm:SisTightOverMultiplicativeNoCallsToell}
  $h_{\times}(x,A) = \left( 1 - |A| + \sum_{i\in A} x_{i} \right)_+$ 
  is the largest possible convex function over the unit cube satisfying the constraints that for integral $x$, $h_{\times}(x,A)\leq 1\ \forall A \subseteq \operatorname{set}(x)$, $h_{\times}(x,A) \leq 0\ \forall A \nsubseteq \operatorname{set}(x)$, and $h_{\times}(x,\operatorname{set}(x)) = 1$, where $(\cdot)_{+} = \max(0,\cdot)$.
  \end{proposition}
\begin{proof}
  Each constraint is an equality constraint or an upper bound, and the values of the function are constrained only at the vertices of the unit cube, so the maximum convex function satisfying the bounds must be the lower hull of the constraint points, that is $h_{\times}(x,A) = \mathbf{C} [ A \subseteq \cdot](x)$, where we have employed Iverson bracket notation in our definition of the function to which we apply $\mathbf{C}$:
  \begin{align}
    \mathbf{C} [ A \subseteq \cdot ](x) =& \min_{\alpha \in \mathbb{R}^{2^{|V|}}} \sum_{B \subseteq V} \alpha_{B} [A\subseteq B] = \min_{\alpha \in \mathbb{R}^{2^{|V|}}} \sum_{B \supseteq A} \alpha_{B}\\
    \text{s.t. }& \sum_{B \subseteq V} \alpha_{B} \mathbf{1}_B = x, \sum_{B\subseteq V} \alpha_B = 1, \alpha_{B} \geq 0\ \forall B .
  \end{align}

  For all $|V| \in \mathbb{Z}_{+}$ and  $A\subseteq V$, we have that
  \begin{align}
    \mathbf{C} [ A \subseteq \cdot ](x) = 0,\  \forall x \in \operatorname{conv}\left( \{y \in \{0,1\}^{|V|} | \operatorname{set}(y) \nsupseteq A \} \right) .
  \end{align}
  For $x \in \operatorname{conv}\left( \{y \in \{0,1\}^{|V|} | \operatorname{set}(y) \supseteq A \lor \exists i \in V, \operatorname{set}(y) \cup \{i\}  \supseteq A \}  \right)$, we must have a linear function that interpolates between $0$ at all integral points $x$ such that $\operatorname{set}(x) \nsupseteq A$ and $1$ at points such that $\operatorname{set}(x)\supseteq A$.  This linear function is exactly $1 - |A| + \sum_{i\in A} x_i$.
\qed\end{proof}
  In practice, we do not need to threshold $h_{\times}$ at zero (i.e.\ we do not need to apply the $(\cdot)_+$ operation above) as this is redundant with the maximization in Equation~\eqref{eq:multiplicativeExtension} being achieved by $A=\emptyset$.
\begin{remark}
  The constraints in Proposition~\ref{thm:SisTightOverMultiplicativeNoCallsToell} are necessary and sufficient conditions on $h$ for this multiplicative family to yield a convex extension for all increasing $\ell$.  They are also necessary and sufficient to yield a convex extension for all $g\in \mathcal{G}_+$.
\end{remark}
\begin{corollary}[$\mathbf{S}_+$ dominates the multiplicative family of extensions]\label{thm:S+dominatesMultiplicativeFamily}
  For all $\mathbf{B}$ that can be written as in Equation~\eqref{eq:multiplicativeExtension} and $h$ having no oracle access to $\ell$,
  \begin{equation}
    \mathbf{B} \leq_{+} \mathbf{S}_+.
    \end{equation}
\end{corollary}

\begin{definition}[$\operatorname{m}g$]
  For a set function $g$,  we may construct a modular function $m$ such that
\begin{equation}\label{eq:ModularFuncEqualSingletons}
m(\{x\}) = g(\{x\}),\ \forall x\in V,
\end{equation}
and denote this function $\operatorname{m}g$.
\end{definition}

For non-increasing supermodular $g$, we may extend the definition of slack rescaling as follows:
\begin{definition}[Slack rescaling for all supermodular $g$]\label{def:SlackRescalingNonincreasing}
Assume $g\in \mathcal{G}$.
For a modular function $m$, we will denote $\operatorname{vec}(m) \in \mathbb{R}^{|V|}$ the vectorization of $m$, and we define
\begin{equation}
\mathbf{S}g(x) = \langle \operatorname{vec}(\operatorname{m}g) , x \rangle + \max_{A\subseteq V} (g(A) - \operatorname{m}g(A)) \left( 1- |A| + \sum_{i\in A} x_i \right).
\end{equation}
\end{definition}
Examples of the convex extension achieved by Definition~\ref{def:SlackRescalingNonincreasing} are given in Figure~\ref{fig:SlackRescalingMinusModularExample}.    
We now show that not only does $\mathbf{S}$ yield an extension for non-increasing supermodular functions, it yields a tighter extension for increasing supermodular functions:
\begin{proposition}\label{prop:SgeqS+}
  $\mathbf{S}_+ \leq_{\mathcal{G}_+} \mathbf{S}$.
\end{proposition}
\begin{proof}
  For $g \in \mathcal{G}_+$,
  \begin{align}
    \mathbf{S} g(x) - \mathbf{S}_+g(x) =&  \max_{A\subseteq V} \left(\langle \operatorname{vec}(\operatorname{m}g) , x \rangle + (g(A) - \operatorname{m}g(A)) \left( 1- |A| + \sum_{i\in A} x_i \right) \right)\\
    &  - \max_{A\subseteq V} g(A)\left( 1 - |A| + \sum_{i\in A} x_{i} \right). \nonumber
  \end{align}
  To show that this difference is greater than zero, we will demonstrate that the difference of each element indexed by $A \subseteq V$ is greater than zero.
  \begin{align}
    \langle \operatorname{vec}(\operatorname{m}g) , x \rangle + (g(A) - \operatorname{m}g(A)) \left( 1- |A| + \sum_{i\in A} x_i \right) 
    - g(A)\left( 1 - |A| + \sum_{i\in A} x_{i} \right) \nonumber \\
    =  \sum_{i\in V} g(\{i\})  x_i  - \sum_{i \in A} g(\{i\}) 
    + \sum_{i\in A} (1-x_i) g(\{i\}) \\
    + \overbrace{\sum_{i\in A} \sum_{j\in A} g(\{i\}) (1-x_j) - \sum_{i\in A} (1-x_i) g(\{i\})}^{=\xi \geq 0} 
    =  \sum_{i\in V \setminus A} g(\{i\})  x_i + \xi \geq 0  \label{eq:SdominatesSplus_TermDiff}
  \end{align}
\qed\end{proof}

\begin{figure}
  \centering
  \subfloat[$\ell(\emptyset) = 0$, $\ell(\{a\}) = 0.5$, $\ell(\{b\}) =  1.5$, $\ell(\{a,b\}) = 4$.]{
    \includegraphics[width=0.45\textwidth]{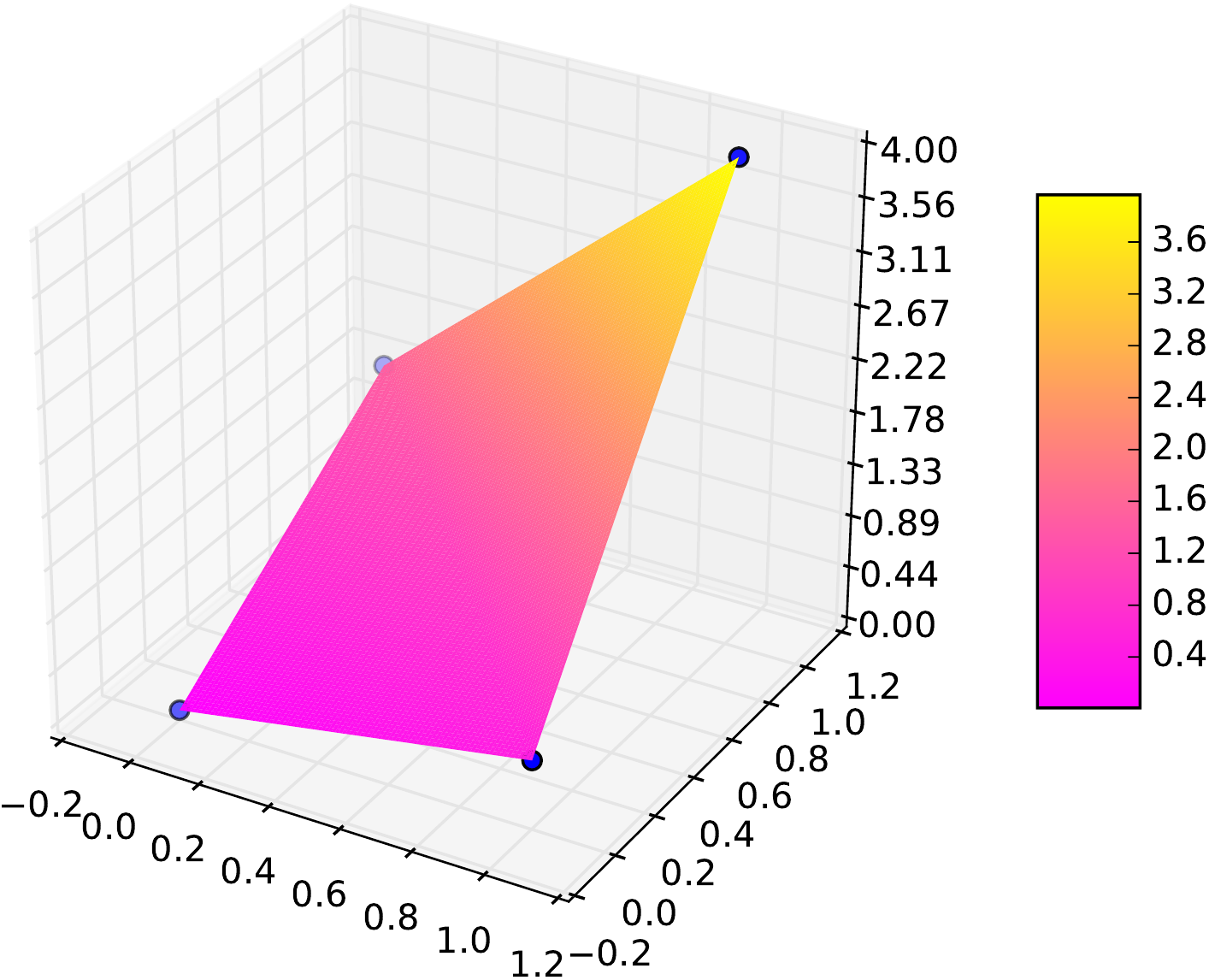}
}
~~
  \subfloat[$\ell(\emptyset) = 0$, $\ell(\{a\}) = -1.5$, $\ell(\{b\}) =  -0.5$, $\ell(\{a,b\}) = 4$.]{
    \includegraphics[width=0.45\textwidth]{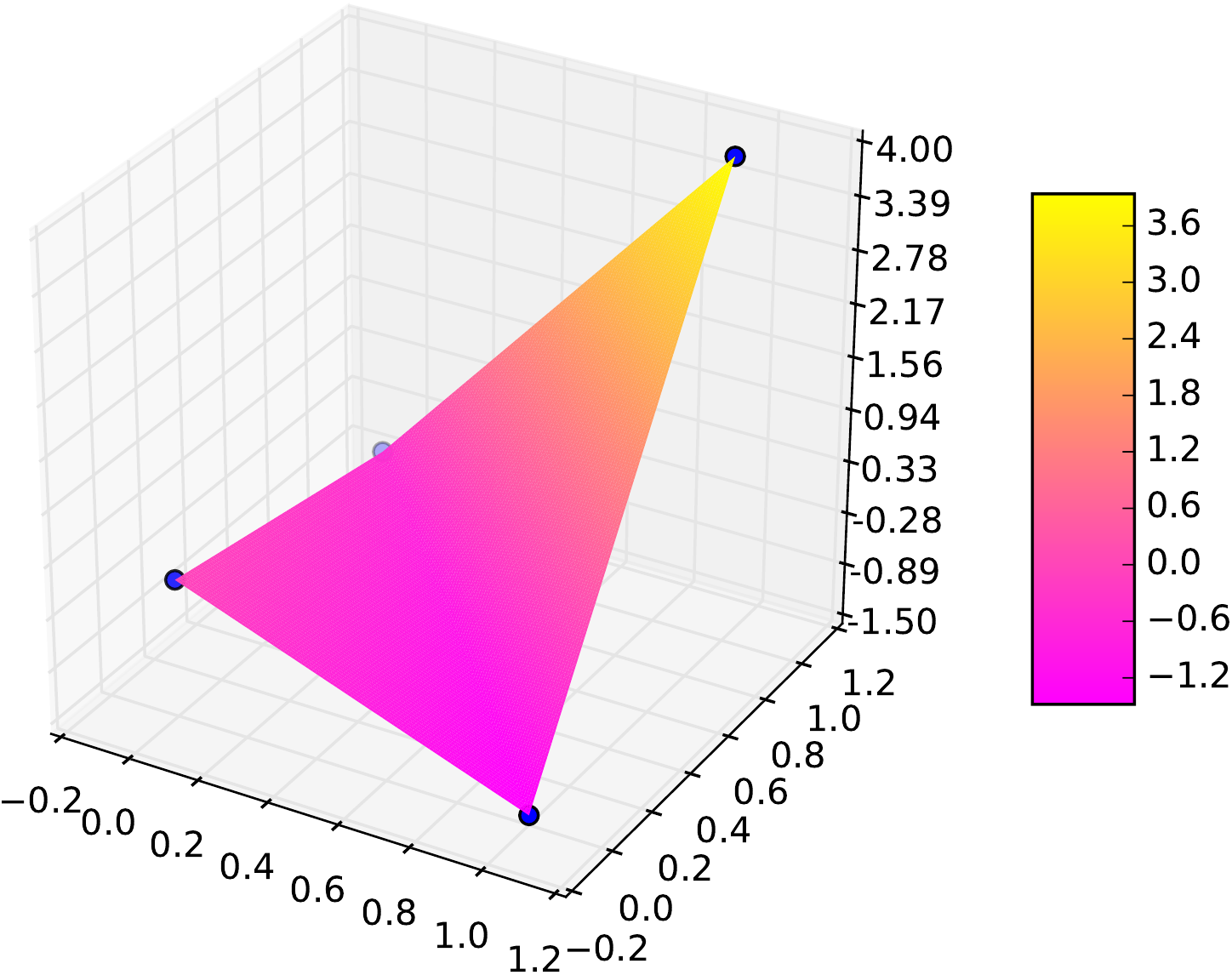}
    }
\caption{An example of the convex extension achieved with Definition~\ref{def:SlackRescalingNonincreasing} applied to positive and non-positive supermodular functions.  In this special case where $|V|=2$, the extension achieves the convex closure.}\label{fig:SlackRescalingMinusModularExample}
\end{figure}

\begin{proposition}
Slack rescaling is invariant to scaling of $g$:
\begin{equation}
\mathbf{S} g =
\gamma^{-1}\mathbf{S} (\gamma g) , \quad \forall \gamma>0 .
\end{equation}
\end{proposition}
\begin{proof}
\begin{align}
\gamma^{-1} \mathbf{S}(\gamma g)(x) =& \frac{1}{\gamma} \left( \langle \operatorname{vec}(\operatorname{m} \gamma g) , x \rangle + \max_{A\subseteq V} (\gamma g(A) - \operatorname{m} \gamma g(A)) \left( 1- \sum_{i\in A} (1-x_i) \right) \right) \\
=& \mathbf{S} g(x)
\end{align}
\qed\end{proof}

\begin{proposition}\label{prop:J2geqS}
  $\mathbf{J}_2 \geq_{\mathcal{G}} \mathbf{S}$.
\end{proposition}
\begin{proof}
  We first add a modular function to $g\in \mathcal{G}$ such that it is increasing. The difference between each of the corresponding $\arg\max$ elements in the definition of $\mathbf{J}_2$ and $\mathbf{S}$, respectively, is $\sum_{i\in V\setminus A} x g(\{i\}) + \sum_{i\in A} (1 - x_i) g(A\setminus \{i\})$.  Each of these terms is non-negative by assumption.
\qed\end{proof}

\begin{proposition}\label{prop:J1ngeqS}
  $\mathbf{J}_1 \ngeq_{\mathcal{G}} \mathbf{S}$, $\mathbf{S} \ngeq_{\mathcal{G}} \mathbf{J}_1$.
\end{proposition}
\begin{proof}
    For $g$ increasing supermodular with $g(\{i\})=0$ $\forall i\in V$, the difference between each of the corresponding $\arg\max$ elements in the definition of $\mathbf{S}$ and $\mathbf{J}_1$, respectively, is $\sum_{i\in A} (1-x_i)(g(V) - g(V\setminus \{i\}) - g(A)) - \sum_{j\in V\setminus A} x_j ( g(A \cup \{j\}) - g(A))$.

  $\mathbf{S} \ngeq_{\mathcal{G}} \mathbf{J}_1$:
Each of the terms in the second sum is non-negative, meaning that the second term is non-positive.  Setting $x_i=1$ for all $i\in A$, $x_j>0$ for all $j\in V\setminus A$ gives the result.

$\mathbf{J}_1 \ngeq_{\mathcal{G}} \mathbf{S}$: Set $x_j=0$ for all $j\in V\setminus A$.  For the difference to be positive, we need to find an increasing $g \in \mathcal{G}$ such that $g(V) > g(A) + g(V\setminus \{i\})$ for some $A\subset V$ and $i\in V$.  Consider the function $g(A) = 0$, $\forall A \subset V$, $g(V) = 1$.
\qed\end{proof}

\subsection{Margin rescaling}

We may similarly define the following convex extension based on margin rescaling: 
\begin{definition}[Margin rescaling for supermodular $g$ \cite{Tsochantaridis2005,yu:hal-01151823}]\label{def:MarginRescalingExtension}
\begin{equation}
\mathbf{M}g(x) := \langle \operatorname{vec}(m) , x \rangle + \max_{A\subseteq V} g(A) - m(A) -|A| +  \sum_{i\in A} x_i.
\end{equation}
\end{definition}
It has been demonstrated that up to a strictly positive scale factor, margin rescaling yields an extension in the convex surrogate loss setting \cite[Proposition 2]{yu:hal-01151823}.  We prove this next and in doing so show necessary and sufficient conditions on the scaling of the loss function for $\mathbf{M}$ to yield an extension.
\begin{proposition}[Proposition 2 of \cite{yu:hal-01151823}]\label{thm:MarginRescalingScaleFactor}
  For every increasing function $\ell$, there exists a scalar $\gamma > 0$ such that $\gamma^{-1} \mathbf{M} (\gamma \ell)$ is an extension of $\ell$.  We will denote $\mathbf{M}_\gamma  :=\gamma^{-1} \mathbf{M} (\gamma \ \cdot)$ for short.
\end{proposition}
\begin{proof}
  Wlog, we may assume that $\ell(\{i\}) = 0,\ \forall i\in V$.  For $\mathbf{M}_{\gamma}$ to yield an extension, we need that for all $A, B \subseteq V$, $\operatorname{set}(x)=A$,
  \begin{equation}\label{eq:necessarySufficientMExtension}
    \gamma g(A) \geq \gamma g(B) - |B| + \sum_{i\in B} x_i = \gamma g(B) - |B| + | A \cap B|.
  \end{equation}
  \begin{align}
    &\gamma(g(A) - g(B)) + |B| - |A\cap B| \geq 0 \\
    &\iff 
     \gamma \leq  \frac{|B| - |A\cap B|}{g(B)-g(A)} \  \forall (A,B) \text{ s.t. } g(A) < g(B) \label{eq:MarginRescalingGammaBound}
  \end{align}
where the restriction to $g(A)<g(B)$ is due to the fact that $g(B)<g(A)$ reverses the direction of the inequality as we multiply both sides by $\frac{1}{g(B)-g(A)}$.    As $g$ is increasing, we cannot simultaneously have that $|A\cap B| = |B|$ and $g(A) < g(B)$ so the numerator is strictly positive and the ratio is therefore strictly positive.
\qed\end{proof}

\begin{proposition}\label{thm:MarginBiggerGammaBetter}
  For $0<\gamma_1 < \gamma_2$ both satisfying the conditions in Equation~\eqref{eq:MarginRescalingGammaBound}, $\mathbf{M}_{\gamma_1} \leq_{+} \mathbf{M}_{\gamma_2}$.
\end{proposition}
\begin{proof}
  Wlog, we will assume that $g \in \mathbf{G}_+$ and ignore the modular transformation in Definition~\ref{def:MarginRescalingExtension}.
  \begin{align}
    0 \leq & \frac{1}{\gamma_{2}} \max_{A\subseteq V} \left( \gamma_{2} g(A) -|A| +  \sum_{i\in A} x_i \right) - \frac{1}{\gamma_1} \max_{A\subseteq V} \left( \gamma_1 g(A) -|A| +  \sum_{i\in A} x_i \right)
    \\
    \impliedby 0 \leq &  \frac{1}{\gamma_{2}} \left( \gamma_{2} g(A) -|A| +  \sum_{i\in A} x_i \right) - \frac{1}{\gamma_1} \left( \gamma_1 g(A) -|A| +  \sum_{i\in A} x_i \right)\\
    =& \underbrace{\left( \frac{1}{\gamma_2} - \frac{1}{\gamma_1} \right)}_{< 0} \underbrace{\left(-|A| + \sum_{i\in A} x_i \right)}_{\leq 0}
  \end{align}
\qed\end{proof}

\begin{proposition}\label{prop:OptimalGammaforM}
  For $g \in \mathcal{G}_+$, the optimal $\gamma$ satisfying Equation~\eqref{eq:MarginRescalingGammaBound} is
  \begin{equation}\label{eq:MarginRescalingOptimalScalefactorSupermodularIncreasing}
    \min_{A,B \subseteq V, g(B)>g(A)} \frac{|B| - |A\cap B|}{g(B) - g(A)} = \left(g(V) - \min_{i\in V} g(V \setminus \{i\})\right)^{-1} .
  \end{equation}
\end{proposition}
\begin{proof}
  For a given $B$ in Eq.~\eqref{eq:MarginRescalingOptimalScalefactorSupermodularIncreasing}, we may show by induction that $A$ must be $B\setminus \{i\}$ for $i$ that maximizes $g(B)-g(A)$. The inductive step can be shown from the property that by supermodularity $g(B\setminus \{i\}) - g(B\setminus \{i,j\}) \leq g(B) - g(B\setminus \{j\}) \leq g(B) - g(B\setminus \{i\}) \implies g(B) - g(B\setminus \{i\}) \geq g(B\setminus \{i\}) - g(B\setminus \{i,j\})$.  This implies $g(B) - g(B\setminus \{i,j\}) \leq 2(g(B) - g(B\setminus \{i\})$, so any $A$ smaller than $|B|-1$ will increase $|B|-|A \cap B|$ more than the increase in the denominator.  From supermodularity, the maximum of $g(B) - g(B\setminus\{i\})$ is achieved for $B=V$.
\qed\end{proof}
\begin{remark}
  We note that Equation~\eqref{eq:MarginRescalingOptimalScalefactorSupermodularIncreasing} bears a strong similarity to the notion of submodular curvature \cite{CONFORTI1984251,NIPS2013_4989,Vondrak2010curvature}, and connections to curvature are interesting directions for future research.
\end{remark}

\begin{proposition}\label{prop:MisPolyTime}
  $\mathbf{M}_\gamma g(x)$ is polynomial time computable for $g\in \mathcal{G}$ polynomial time computable and $x \in [0,1]^{|V|}$.
\end{proposition}
\begin{proof}
  Computation of $\operatorname{m} g$ and $\gamma$ require a linear number of calls to $g$.  The $\arg\max$ in the definition of $\mathbf{M}_{\gamma}$ is a maximization of the sum of supermodular and modular functions.
\qed\end{proof}

As in Equation~\eqref{eq:multiplicativeExtension}, we may view margin rescaling as a special case of an additive strategy for computing convex extensions in which
\begin{equation}\label{eq:additiveExtension}
\mathbf{B} \ell(x) = \frac{1}{\gamma} \max_{A \subseteq V} \left( \gamma \ell(A) + h_{+}(x,A) \right)
\end{equation}
where as in Equation~\eqref{eq:multiplicativeExtension}, $h_{+}$ is some convex function.  We have seen from Proposition~\ref{thm:MarginRescalingScaleFactor} that $\gamma$ is necessary to guarantee that the operator yields an extension for members of this family.  We again restrict ourselves to $h_{+}$ that do not have access to $\ell$ except the assumption that it is an increasing function satisfying scaling constraints. 
\begin{proposition}\label{thm:MarginDominatesAdditive}
Margin rescaling, $\mathbf{M}_{\gamma}$, dominates the family of additive extensions following Equation~\eqref{eq:additiveExtension}, where $h_{+}(x,A)$ has no oracle access to $\ell$.
\end{proposition}
\begin{proof}
For Equation~\eqref{eq:additiveExtension} to yield an extension, we require that $\gamma \ell(A) + h_{+}(x,A) \geq \gamma \ell(\operatorname{set}(x))$ for integral $x \in [0,1]^{|V|}$, which is satisfied for $h_{+}(x,A) \geq \gamma (\ell(\operatorname{set}(x)) - \ell(A))$.  For increasing $\ell$, $\gamma (\ell(\operatorname{set}(x)) - \ell(A))$ is bounded below by $0$ for all $A \subseteq \operatorname{set}(x)$, and for $A \nsubseteq \operatorname{set}(x)$ we have from Equation~\eqref{eq:necessarySufficientMExtension} that $\gamma(\ell(\operatorname{set}(x)) - \ell(A)) \geq -|A| + |\operatorname{set}(x) \cap A|$.
These yield a set of constraints similar to those in Proposition~\ref{thm:SisTightOverMultiplicativeNoCallsToell} but shifted downwards by one, and without thresholding to zero. In Proposition~\ref{thm:SisTightOverMultiplicativeNoCallsToell} the maximal convex function satisfying the constraints had the form $\left(1 -|A| + \sum_{i\in A} x_i\right)_+$, and we may conclude that the optimal function for the additive family of extensions is achieved by $h_{+}(x,A) = -|A| + \sum_{i\in A} x_i$.
This indicates that $\mathbf{M}_\gamma \ell$ with maximal $\gamma$ satisfying Equation~\eqref{eq:MarginRescalingGammaBound} (Proposition~\ref{thm:MarginBiggerGammaBetter}) is the additive convex extension closest to the convex closure of $\ell$.
\qed\end{proof}

\begin{proposition}\label{prop:J1geqM}
  $\mathbf{J}_1 \geq_{\mathcal{G}} \mathbf{M}_\gamma$.
\end{proposition}
\begin{proof}
    For $g$ increasing supermodular with $g(\{i\})=0$ $\forall i\in V$ and scaling satisfying Proposition~\ref{prop:OptimalGammaforM}, the difference between each of the corresponding $\arg\max$ elements in the definition of $\mathbf{M}$ and $\mathbf{J}_1$, respectively, is
 $ -\sum_{i \in V\setminus A} x_i \left( g(A\cup \{i\}) - g(A) \right)  + \sum_{i\in A} (1-x_i) \left( g(V) - g(V\setminus \{i\}) - 1 \right)$.  The first summation is non-positive as $x_i\geq 0$ and $g$ is increasing.  The second term is non-positive as $g(V) - g(V\setminus \{i\}) \leq 1$ from Proposition~\ref{prop:OptimalGammaforM}.
\qed\end{proof}

\begin{proposition}\label{prop:J2geqM}
  $\mathbf{J}_2 \geq_{\mathcal{G}} \mathbf{M}_\gamma$.
\end{proposition}
\begin{proof}
 For $g$ increasing supermodular and appropriately scaled, the difference between each of the corresponding $\arg\max$ elements in the definition of $\mathbf{M}$ and $\mathbf{J}_2$, respectively, is $\sum_{i\in A} (1-x_i)(g(A) - g(A\setminus \{i\}) - 1) - \sum_{j\in V\setminus A} x_j g(\{i\})$.  The second term is clearly non-positive, and the first term is non-positive if $g(A) - g(A\setminus \{i\}) \leq 1$, which follows from Proposition~\ref{prop:OptimalGammaforM} and the supermodularity of $g$. 
\qed\end{proof}

We next demonstrate that neither margin rescaling nor slack rescaling dominate the other, and neither is guaranteed to be closer to the convex closure everywhere in the unit cube.
\begin{proposition}\label{thm:SlackNotDominatesMargin}
Neither slack rescaling nor margin rescaling dominates the other: $ \mathbf{M}_{\gamma} \nleq_{\mathcal{G}} \mathbf{S}$, $\mathbf{S} \nleq_{\mathcal{G}} \mathbf{M}_{\gamma}$.
\end{proposition}
\begin{proof}
  Define $g$ to be a symmetric set function with $V = \{a,b,c,d\}$ and
$g(\emptyset) = 0,\ g(\{a\}) = -0.5,\ g(\{a,b\}) = -\frac{2}{3},\ g(\{a,b,c\}) = -0.5,\ g(\{a,b,c,d\}) = 0$.
It is straightforward to show numerically that $\mathbf{S}g(x) > \mathbf{M}_{\gamma}g(x)$ in the neighborhoods around where $|\operatorname{set}(x)|=2$, and that $\mathbf{M}_{\gamma}g(x) > \mathbf{S}g(x)$ in the neighborhood around where $|\operatorname{set}(x)|=4$.
\qed\end{proof}

\section{Experimental validation}\label{sec:Empirical}

In this section, we validate the theoretical results by optimizing an objective for exemplar selection from a set of images.  The objective to be optimized is
\begin{equation}\label{eq:CoverageObjectiveBudget}
  \arg\min_{A \subseteq V} -\sum_{i\in{V}} \left[ \min_{j\in A} \| y_i - y_j \| \leq \varepsilon \right], \text{ s.t.\ } |A|\leq C ,
\end{equation}
which is supermodular as it is simply the negative of a coverage function.
We have optimized this objective for a random subsample of images from \cite{Torralba2008MTI}, where $\varepsilon$ was set to the 90\% percentile of all pairwise distances.  For the LP relaxations, the discrete constraint $|A| \leq C$ was relaxed to $\|x\|_1 \leq C$.  Discretization of LP relaxations was achieved by setting $A$ to contain the $C$ largest indices of $x$.  LP relaxations began with the tightest budget constraint, and each increased budget was initialized with the cutting planes of the previous iteration.  Results are shown in Figure~\ref{fig:CoverageResults}.
\begin{figure}[t]
  \centering
  \includegraphics[width=0.6\textwidth]{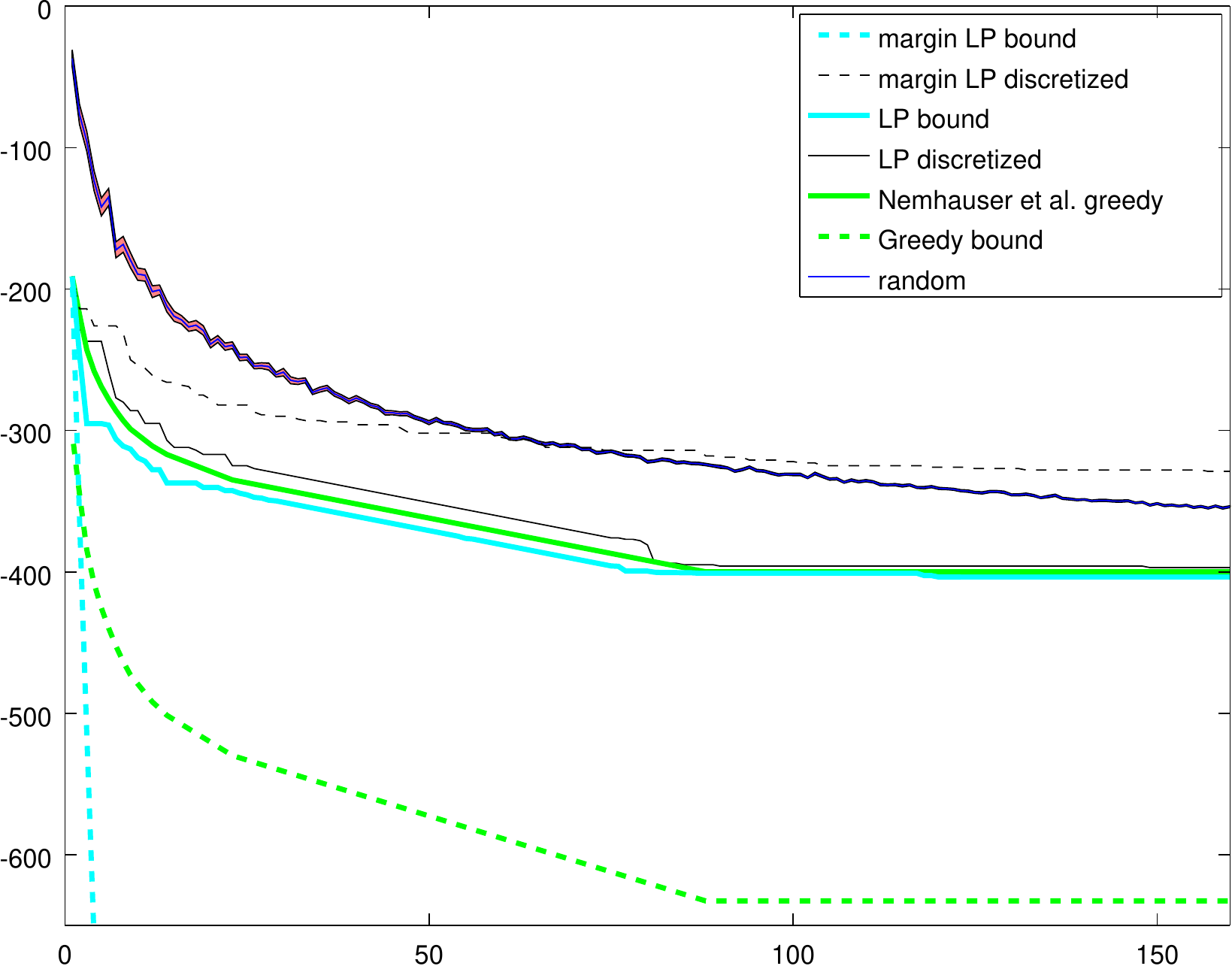}
  \caption{Empirical results optimizing Equation~\eqref{eq:CoverageObjectiveBudget} on a sample of 400 images, where the horizontal axis represents the budget $C$.}\label{fig:CoverageResults}
\end{figure}
 The margin rescaling LP bound is very loose, at a budget of $C=160$ the bound is $-8195.7$ although the minimum of the coverage objective is $-400$.  By contrast, the LP that incorporates the constraints from Definition~\ref{def:modularBoundsJegelkaIyer} gives very tight optimality bounds close to the empirical performance of the Nemhauser et al.\ greedy algorithm, and a huge improvement over the known $(1-1/e)$ multiplicative offline bound \cite{Nemhauser1978}.
Minimization for a single budget took a few seconds on a MacBook Pro for the LP incorporating constraints from $\mathbf{J}_1$ and $\mathbf{J_2}$, while a single budget could take tens of minutes for minimization of $\mathbf{M}_\gamma$.

\section{Discussion and Conclusions}

In this work, we have formally analyzed slack and margin rescaling as convex extensions of supermodular set functions.  Although they were originally developed for empirical risk minimization of (increasing) functions, we show that supermodularity can be exploited to first transform a set function into an increasing function with a linear number of accesses to the loss function.  We have shown that neither slack rescaling nor margin rescaling dominates the other in the sense that it is strictly closer to the convex closure.  However, computation of slack rescaling for a supermodular function corresponds to a non-supermodular maximization problem.  Margin rescaling, by contrast, remains tractable.  We have further shown that margin and slack rescaling correspond to optimal additive and multiplicative extensions, respectively, given a computational budget of one oracle access to the loss function.  Still, slack and margin rescaling are dominated by extensions derived from known modular bounds on submodular functions (Definition~\ref{def:modularBoundsJegelkaIyer}).  Empirically, incorporating cutting planes from these extensions gives nearly optimal bounds on an objective that measures coverage on a budget, indicating that this approach may be of wider interest for obtaining very tight empirical optimality guarantees for submodular maximization.  At the same time, this indicates that novel structured prediction algorithms with improved performance over structured output SVMs may be developed based on the extensions in Definition~\ref{def:modularBoundsJegelkaIyer}.  Complete source code for the experiments is available for download from \url{https://github.com/blaschko/supermodularLP}.

\subsubsection*{Acknowledgments}

This work is funded by Internal Funds KU Leuven, FP7-MC-CIG 334380, and the Research Foundation - Flanders (FWO) through project number G0A2716N.  

\bibliographystyle{abbrv}
\bibliography{biblio}

\end{document}